\documentclass{article}

\usepackage{iclr2019_conference,times}
\usepackage{amsmath}
\usepackage{amssymb}
\usepackage{amsthm}

\usepackage[utf8]{inputenc} 
\usepackage[T1]{fontenc}    
\usepackage[hidelinks]{hyperref}       
\usepackage{url}            
\usepackage{booktabs}       
\usepackage{amsfonts}       
\usepackage{nicefrac}       
\usepackage{microtype}      

\usepackage{stmaryrd}
\usepackage{wrapfig}
\usepackage{tikz}
\usepackage{xspace}

\newcommand{\ospace}{\mathcal{X}}
\newcommand{\obs}{x}
\newcommand{\espace}{\mathcal{D}}
\newcommand{\eprims}{\mathcal{D}_0}
\newcommand{\oracle}{D}
\newcommand{\expr}{d}
\newcommand{\ecomp}[2]{\langle#1, #2\rangle}

\newcommand{\edist}{\Delta}
\newcommand{\einterp}{\hat{f}}
\newcommand{\eparam}{\eta}

\newcommand{\rspace}{\Theta}
\newcommand{\repr}{\theta}
\newcommand{\rcomp}[2]{#1 * #2}

\newcommand{\rdist}{\delta}

\newcommand{\fn}{f}

\newcommand{\indicate}{\mathbb{I}}
\newcommand{\err}{\textsc{tre}\xspace}

\newcommand{\eps}{\epsilon}
\newcommand{\cnn}{\mathtt{CNN}}

\DeclareMathOperator*{\argmin}{arg\,min}

\newtheorem{defn}{Definition}
\newtheorem{prop}{Proposition}
\newtheorem{lemma}{Lemma}

\newtheorem{remark}{Remark}

\newsavebox{\bx}
\newenvironment{mybox}{
  \setlength{\fboxsep}{4mm}
  \begin{lrbox}{\bx}
    \begin{minipage}{0.925\textwidth}
}{
    \end{minipage}
  \end{lrbox}
  \vspace{2mm}
  \fbox{
    \usebox{\bx}
  }
  \vspace{2mm}
}
\newcommand{\myboxtitle}[1]{\textbf{#1}\\\par}

\newcommand{\eg}{e.g.\ }

\iclrfinalcopy

\title{Measuring Compositionality in \\ Representation Learning}

\author{
\strut \hspace{-1.5mm} Jacob Andreas \\
Computer Science Division \\
University of California, Berkeley \\
\texttt{jda@cs.berkeley.edu} \\
}

\begin{document}

\maketitle

\begin{abstract}
  Many machine learning algorithms represent input data with vector embeddings
  or discrete codes.  When inputs exhibit compositional structure (\eg objects
  built from parts or procedures from subroutines), it is natural to ask whether
  this compositional structure is reflected in the the inputs' learned
  representations. While the assessment of compositionality in languages has
  received significant attention in linguistics and adjacent fields, the machine
  learning literature lacks general-purpose tools for producing graded
  measurements of compositional structure in more general (\eg vector-valued)
  representation spaces. We describe a procedure for evaluating compositionality
  by measuring how well the true representation-producing model can be
  approximated by a model that explicitly composes a collection of inferred
  representational primitives. We use the procedure to provide formal and
  empirical characterizations of compositional structure in a variety of
  settings, exploring the relationship between compositionality and learning
  dynamics, human judgments, representational similarity, and generalization.
\end{abstract}

\section{Introduction}

\begin{wrapfigure}{r}{0.5\textwidth}
  \vspace{-1.2em}
  \centering
  \strut
    \includegraphics[width=0.47\textwidth,clip,trim=1cm 14cm 11.9cm .5cm]{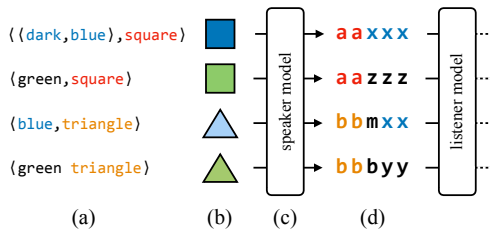}
  \vspace{-.4em}
	\caption{Representations arising from a communication game. In this game, an
  observation (b) is presented to a learned speaker model (c), which
  encodes it as a discrete character sequence (d) to be consumed by a listener
  model for some downstream task.
  The space of inputs has known compositional structure (a).
  We want to measure the extent to which this structure is reflected
  (perhaps imperfectly) in the structure of the learned codes.
  }
  \label{fig:teaser}
\end{wrapfigure}

The success of modern representation learning techniques
has been accompanied by an interest in understanding 
the structure of learned representations. One feature shared by many
\emph{human}-designed representation systems is compositionality: the capacity
to represent complex concepts (from objects to procedures to beliefs) by
combining simple parts \citep{Fodor02Compositionality}. While many machine
learning approaches make use of human-designed compositional analyses for
representation and prediction \citep{Socher13CVG,Dong16Seq2SeqSemparse}, it is
also natural to ask whether (and how) compositionality arises in learning
problems where compositional structure has not been built in from the start.
Consider the example in \autoref{fig:teaser}, which shows a hypothetical
character-based encoding scheme learned for a simple communication task (similar
to the one studied by \citeauthor{Lazaridou16Communication},
\citeyear{Lazaridou16Communication}).  Is this encoding scheme compositional?
That is, to what extent can we analyze the agents' messages as being built from
smaller pieces (\eg pieces \texttt{xx} meaning \emph{blue} and \texttt{bb}
meaning \emph{triangle})?

A large body of work, from early experiments on language evolution to recent
deep learning models \citep{Kirby98Learning,Lazaridou16LanguageGame}, aims to
answer questions like this one. But existing solutions rely on manual (and often
subjective) analysis of model outputs \citep{Mordatch17Emergence}, or at best
automated procedures tailored to the specifics of individual problem domains
\citep{Brighton06Topo}.  They are difficult to compare and difficult to apply
systematically.

We are left with a need for a standard, formal, automatable and quantitative
technique for evaluating claims about compositional structure in learned
representations. The present work aims at first steps toward meeting that need.
We focus on an \emph{oracle} setting where the compositional structure of model inputs is
known, and where the only question is whether this structure is reflected in
model outputs. This oracle evaluation paradigm covers most of the existing
representation learning problems in which compositionality has been studied.  

The first contribution of this paper is a simple formal framework for measuring
how well a collection of representations (discrete- or continuous-valued)
reflects an oracle compositional analysis of model inputs. We propose an
evaluation metric called \err, which provides graded judgments of
compositionality for a given set of (input, representation) pairs.  The core of
our proposal is to treat a set of primitive meaning representations as
hidden, and optimize over them to find an explicitly compositional model that
approximates the true model as well as possible. 
For example, if the compositional structure that describes an object is a simple
conjunction of attributes, we can search for a collection of ``attribute
vectors'' that sum together to produce the observed object representations; if
it is a sparse combination of (attribute, value) pairs we can additionally
search for ``value vectors'' and parameters of a binding operation; and so on
for more complex compositions.

Having developed a tool for assessing the compositionality of representations,
the second contribution of this paper is a survey of applications. We
present experiments and analyses aimed at answering four questions about the
relationship between compositionality and learning:
\begin{itemize}
  \item How does compositionality of representations evolve in relation to other
    measurable model properties over the course of the learning process?
    (\autoref{sec:bottleneck})
  \item How well does compositionality of representations track human judgments
    about the compositionality of model inputs? (\autoref{sec:humans})
  \item How does compositionality constrain distances between representations,
    and
    how does \err relate to other methods that analyze representations based
    on similarity? (\autoref{sec:similarity})
  \item Are compositional representations necessary for generalization to
    out-of-distribution inputs? (\autoref{sec:generalization})
\end{itemize}
We conclude with a discussion of possible applications and generalizations
of \err-based analysis.

\section{Related Work}

Arguments about whether distributed (and other non-symbolic) representations
could model compositional phenomena were a staple of 1980s-era
connectionist--classicist debates. \citet{Smolensky98Constituency} provides an
overview of this discussion and its relation to learnability, as well as a
concrete implementation of a compositional encoding scheme with distributed
representations. Since then, numerous other approaches for compositional
representation learning have been proposed, with
\citep{Mitchell08Composition,Irsoy2014DeepTreeRnn} and without
\citep{Dircks99LanguageChange,Havrylov17Comm} the scaffolding of explicit
composition operations built into the model.

The main experimental question is thus when and how compositionality arises
``from scratch'' in the latter class of models. In order to answer this
question it is first necessary to determine whether compositional structure is
present at all. Most existing proposals come from linguistics and and
philosophy, and offer evaluations of compositionality targeted
at analysis of formal and natural languages
\citep{Carnap37Logical,Lewis76GeneralSemantics}. Techniques from this literature
are specialized to the details of linguistic representations---particularly the
algebraic structure of grammars \citep{Montague70UG}. It is not straightforward
to apply these techniques in more general settings, particularly those featuring
non-string-valued representation spaces. We are not aware of existing work that
describes a procedure suitable for answering questions about compositionality in
the general case.

Machine learning research has responded to this absence in several ways.
One class of evaluations \citep{Mordatch17Emergence,Choi18Obverter} derives
judgments from ad-hoc manual analyses of representation spaces.  These analyses
provide insight into the organization of representations but
are time-consuming and non-reproducible. Another class of evaluations
\citep{Brighton02Transmission,Andreas17RNNSyn,Bogin18ConsistentSpeakers}
exploits task-specific structure (\eg the ability to elicit pairs of
representations known to feature particular relationships) to give
evidence of compositionality.  Our work aims to provide a standard and
scalable alternative to these model- and task-specific evaluations.

Other authors refrain from measuring compositionality directly, and instead base
analysis on measurement of related phenomena, for which more standardized
evaluations exist. Examples include correlation between representation
similarity and similarity of oracle compositional analyses
\citep{Brighton06Topo} and generalization to structurally novel inputs
\citep{Kottur17Takedown}. Our approach makes it possible to examine the
circumstances under which these surrogate measures in fact track stricter
notions of compositionality; 
similarity is discussed in Sec.\ \ref{sec:similarity}
and generalization in Sec.\ \ref{sec:generalization}.

A long line of work in natural language processing
\citep{Coecke10Tensor,Baroni10Matrices,Clark12VSM,Fyshe15Compositional} focuses
on learning composition functions to produce distributed representations of
phrases and sentences---that is, for purposes of modeling rather than
evaluation. We use one experiment from this literature to validate our own
approach (\autoref{sec:humans}).  On the whole, we view work on compositional
representation learning in NLP as complementary to the framework presented here:
our approach is agnostic to the particular choice of composition function, and
the aforementioned references provide well-motivated choices suitable for
evaluating data from language and other sources.
Indeed, one view of the present work is simply as a demonstration that we can
take existing NLP techniques for compositional representation learning, fit them
to representations produced by other models (even in non-linguistic settings),
and view the resulting \emph{training loss} as a measure of the compositionality of the
representation system in question.

\section{Evaluating Compositionality}
\label{sec:background}
\label{sec:tre}

Consider again the communication task depicted in \autoref{fig:teaser}. Here, a
speaker model observes a target object described by a feature vector. The
speaker sends a message to a listener model, which uses the message to complete
a downstream task---for example, identifying the referent from a collection of
distractors based on the content of the message
\citep{Enquist94Symmetry, Lazaridou16LanguageGame}.  Messages produced by the
speaker model serve as \emph{representations} of input objects; we want to know
if these representations are compositional. Crucially, we may already know
something about the structure of the inputs themselves. In this example, inputs
can be identified via composition of categorical shape and color attributes. How
might we determine whether this \emph{oracle} analysis of input structure is
reflected in the structure of representations? This section proposes an
automated procedure for answering the question.

\paragraph{Representations} A representation learning problem is defined by
    a dataset $\ospace$ of \emph{observations} $\obs$
    (\autoref{fig:teaser}b);
    a space $\rspace$ of \emph{representations} $\repr$
    (\autoref{fig:teaser}d);
    and
    a \emph{model} $\fn : \ospace \to \rspace$
    (\autoref{fig:teaser}c)%
.
We assume that the representations produced by $\fn$ are used in a larger
system to accomplish some concrete task, the details of which are not important
for our analysis.

\paragraph{Derivations} The technique we propose additionally assumes we
have prior knowledge about the compositional structure of inputs. In particular,
we assume that inputs can be labeled with tree-structured
\emph{derivations} $\expr$ (\autoref{fig:teaser}a), defined by
    a finite set $\eprims$ of \emph{primitives}
    and
    a binary \emph{bracketing} operation $\ecomp{\cdot}{\cdot}$, such
    that if
    $\expr_i$ and $\expr_j$ are derivations, $\ecomp{\expr_i}{\expr_j}$ is a
    derivation%
.
Derivations are produced by
    a \emph{derivation oracle} $\oracle : \ospace \to \espace$.

\paragraph{Compositionality}
In intuitive terms, the representations computed by $\fn$ are compositional
if each $\fn(x)$ is determined by the structure of $\oracle(x)$.
Most discussions of compositionality, following \citet{Montague70UG}, make this
precise by defining
    a \emph{composition} operation $\rcomp{\repr_a}{\repr_b} \mapsto \repr$ in
    the space of representations.
Then the model $\fn$ is compositional if it is a homomorphism from
inputs to representations: we require that for any $\obs$ with
$\oracle(\obs) = \ecomp{\oracle(\obs_a)}{\oracle(\obs_b)}$,
\begin{equation}
  \label{eq:montague}
  \fn(\obs) = \rcomp{\fn(\obs_a)}{\fn(\obs_b)} \ .
\end{equation}

In the linguistic contexts for which this definition was originally proposed, it
is straightforward to apply.  Inputs $\obs$ are natural language strings. Their
associated derivations $\oracle(\obs)$ are syntax trees, and composition of
derivations is syntactic composition. Representations $\repr$ are logical
representations of meaning (for an overview see
\citeauthor{vanBenthem96Handbook}, \citeyear{vanBenthem96Handbook}). To argue
that a particular fragment of language is compositional, it is sufficient to
exhibit a \emph{lexicon} $\eprims$ mapping words to their associated meaning
representations, and a \emph{grammar} for composing meanings where licensed by
derivations. Algorithms for learning grammars and lexicons from data are a
mainstay of semantic parsing approaches to language understanding
problems like question answering and instruction following
\citep{Zettlemoyer05CCG,Chen12Online,Artzi14Compact}.

But for questions of compositionality involving more general representation
spaces and more general analyses, the above definition presents two
difficulties:  (1) In the absence of a clearly-defined syntax of the kind
available in natural language, how do we identify lexicon entries: the
primitive parts from which representations are constructed? (2) What do we do
with languages like the one in \autoref{fig:teaser}d, which seem to exhibit
\emph{some} kind of regular structure, but for which the homomorphism condition
given in \autoref{eq:montague} cannot be made to hold exactly?

Consider again the example in \autoref{fig:teaser}. The oracle derivations tell
us to identify primitive representations for \emph{dark}, \emph{blue}, \emph{green},
\emph{square}, and \emph{triangle}. The derivations then suggest a process for
composing these primitives (\eg via string concatenation) to produce full
representations. The speaker model is compositional (in the sense of
\autoref{eq:montague}) as long as there is \emph{some} assignment of
representations to primitives such that for each model input, composing
primitive representations according to the oracle derivation reproduces the
speaker's prediction. 

In \autoref{fig:teaser} there is no assignment of strings to primitives that
reproduces model predictions exactly. But predictions can be reproduced
approximately---by taking \texttt{xx} to mean \emph{blue}, \texttt{aa} to mean
\emph{square}, etc. The quality of the approximation serves as a measure of the
compositionality of the true predictor: predictors that are mostly compositional
but for a few exceptions, or compositional but for the addition of some noise,
will be well-approximated on average, while arbitrary mappings from inputs to
representations will not.
This suggests that we should measure compositionality by searching for
representations that allow an explicitly compositional model to approximate the
true $\fn$ as closely as possible.  We define our evaluation procedure as
follows:
\vspace{.5em}

\begin{mybox}\myboxtitle{Tree Reconstruction Error (\err)}

  First choose :
  \begin{itemize}
    \item a distance function $\rdist : \rspace \times \rspace \to [0, \infty)$ satisfying $\rdist(\repr, \repr') = 0 \Leftrightarrow \repr = \repr'$
    \item a composition function $\rcomp{}{} : \rspace \times \rspace \to \rspace$
  \end{itemize}
Define 
$\einterp_\eparam(\expr)$,
a \emph{compositional approximation to $\fn$} with parameters $\eparam$,
as:
\begin{align*}
  \einterp_\eparam(\expr_i) 
    &= \eparam_i 
    && \textrm{for $\expr_i \in \eprims$} \\
  \einterp_\eparam\big(\ecomp{\expr}{\expr'}\big) 
    &= \rcomp{\einterp_\eparam(\expr)}{\einterp_\eparam(\expr')}
    && \textrm{for all other $\expr$}
\end{align*}
$\einterp_\eparam$ has one parameter vector $\eparam_i$ for every $\expr_i$ in
$\eprims$; these vectors are members of the representation space $\rspace$. \\

Given a dataset $\ospace$ of inputs $\obs_i$ with derivations $\expr_i =
\oracle(\obs_i)$, compute:
\begin{align}
  \label{eq:treopt}
  \eparam^* &= \argmin_\eparam \ \sum_i \rdist\big(\fn(\obs_i),
  \einterp_\eparam(\expr_i)\big) \\
  \intertext{Then we can define datum- and dataset-level evaluation metrics:}
  \label{eq:tre}
  \err(\obs) &= \rdist\big(\fn(\obs), \einterp_{\eparam^*}(\expr)\big) \\
  \err(\ospace) &= \frac{1}{n} \sum_i \err(\obs_i)
\end{align}
\vspace{-.5em}

\vspace{-1em}
\end{mybox} 
\vspace{-1em}

\paragraph{\err and compositionality} How well does the evaluation metric $\err(\ospace)$ capture the intuition behind
\autoref{eq:montague}? The definition above uses parameters $\eta_i$ to witness
the constructability of representations from parts, in this case by explicitly
optimizing over those parts rather than taking them to be given by $\fn$.  Each
term in \autoref{eq:treopt} is analogous to an instance of
\autoref{eq:montague}, measuring how well $\einterp_{\eparam^*}(\obs_i)$, the best
compositional prediction, matches the true model prediction $\fn(\obs_i)$. In
the case of models that are homomorphisms in the sense of \autoref{eq:montague},
\err reduces to the familiar case: \pagebreak

\begin{remark}
  $\err(\obs) = 0$ for all $\obs$ if and only if \autoref{eq:montague} holds
  exactly (that is, $\fn(\obs) = \rcomp{\fn(\obs_a)}{\fn(\obs_b)}$ for any
  $\obs, \obs_a, \obs_b$ with $\oracle(\obs) =
  \ecomp{\oracle(\obs_a)}{\oracle(\obs_b)}$).
\end{remark}\vspace{-1em}
\begin{proof}
  One direction follows immediately from defining $\einterp_{\eparam^*}(x) =
  \fn(x)$. For the other,
  $
    \fn(\obs) = \einterp(\oracle(\obs)) 
    = \einterp(\ecomp{\oracle(\obs_a)}{\oracle(\obs_b)}) 
    = \rcomp{\einterp(\oracle(\obs_a))}{\einterp(\oracle(\obs_b))} 
    = \rcomp{\fn(\obs_a)}{\fn(\obs_b)}
  $.
\end{proof}

\paragraph{Learnable composition operators} The definition of \err leaves the
choice of $\rdist$ and $\rcomp{}{}$ up to the evaluator.  Indeed, if the exact form
of the composition function is not known $\emph{a priori}$, it is natural to
define $*$ with free parameters (as in e.g.\ \citeauthor{Baroni10Matrices},
\citeyear{Baroni10Matrices}), treat these as another learned part of $\einterp$,
and optimize them jointly with the $\eparam_i$. However, some care must be taken
when choosing $\rcomp{}{}$ (especially when learning it) to avoid trivial
solutions:\\[-0.5em]

\begin{remark}
  Suppose $\oracle$ is injective; that is, every $\obs \in \ospace$ is assigned
  a unique derivation. Then there is always some $\rcomp{}{}$ that achieves
  $\err(\ospace) = 0$: simply \emph{define} $\rcomp{\fn(\obs_a)}{\fn(\obs_b)} =
  \fn(\obs)$ for any $\obs, \obs_a, \obs_b$ as in the preceding definition, and
  set $\einterp = \fn$.
\end{remark}
In other words, some pre-commitment to a restricted composition function is
essentially inevitable: if we allow the evaluation procedure to select an
arbitrary composition function, the result will be trivial.  This paper features
experiments with $\rcomp{}{}$ in both a fixed functional form and a learned
parametric one.

\paragraph{Implementation details}

For models with continuous $\rspace$ and differentiable $\rdist$ and
$\rcomp{}{}$, $\err(\ospace)$ is also differentiable. \autoref{eq:treopt} can be
solved using gradient descent. We use this strategy in Sections
\ref{sec:bottleneck} and \ref{sec:humans}. For discrete $\rspace$, it may be
possible to find a continuous \emph{relaxation} with respect to which
$\rdist(\repr, \cdot)$ and $\rcomp{}{}$ are differentiable, and gradient descent
again employed. We use this strategy in \autoref{sec:generalization} (discussed
further there).  An implementation of an SGD-based \err solver is provided in
the accompanying software release. For other problems, task-specific optimizers
(e.g.\ machine translation alignment models;
  \citeauthor{Bogin18ConsistentSpeakers}, \citeyear{Bogin18ConsistentSpeakers})
  or general-purpose discrete optimization toolkits can be applied to
  \autoref{eq:treopt}. \\[-0.5em]

The
remainder of the paper highlights ways of using \err to answer questions about
compositionality that arise in machine learning problems of various kinds. 

\section{Compositionality and Learning Dynamics}
\label{sec:bottleneck}

We begin by studying the relationship between compositionality and learning
dynamics, focusing on the information bottleneck theory of representation
learning proposed by \citet{Tishby15Info}. This framework proposes that
learning in deep models consists of an error minimization phase followed by a
compression phase, and that compression is characterized by a decrease in the
mutual information between inputs and their computed representations. We
investigate the hypothesis that the compression phase finds a
\emph{compositional} representation of the input distribution, isolating
decision-relevant attributes
and discarding irrelevant information.

\begin{wrapfigure}{r}{0.58\textwidth}
  \vspace{-1.4em}
  \centering
  \strut
  \includegraphics[width=0.55\textwidth,clip,trim=.5cm 12.3cm 8cm .5cm]{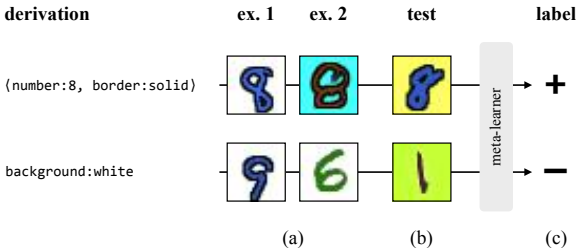}
  \vspace{-.15em}
  \caption{
    Meta-learning task: learners are presented with two example images
    depicting a visual concept (a), and must determine whether a third
    image (b) is an example of the same concept (c). 
  }
  \label{fig:meta}
  \vspace{-.5em}
\end{wrapfigure}

Data comes from a few-shot classification task.  Because our analysis focuses on
compositional hypothesis classes, we use visual concepts from the Color MNIST
dataset of \citet{Seo17VisDial} (\autoref{fig:meta}). We predict classifiers in
a meta-learning framework \citep{Schmidhuber87Thesis, Santoro16MANN}: for each
sub-task, the learner is presented with two images corresponding to some
compositional visual concept (\eg ``digit 8 on a black background'' or ``green
with heavy stroke'') and must determine whether a held-out image is an example
of the same visual concept.

Given example images $\obs_1$ and $\obs_2$, a test image $\obs^*$, and label
$y^*$, the model computes:
\begin{align*}
  z_i &= \cnn(\obs_i) \textrm{ for } i \in \{1, 2, *\} \\
  \repr &= \tanh(W (z_1 + z_2)) \\
  \hat{y} &= \repr^\top z_t
\end{align*}
We use $\theta$ as the representation of a classifier for analysis.
The model is trained to minimize the logistic loss between logits $\hat{y}$ and
ground-truth labels $y^*$. More details are given in \autoref{app:models}.

\paragraph{Compositional structure}

Visual concepts used in this task are all single attributes or conjunctions of
attributes; i.e.\ their associated derivations are of the form $\texttt{attr}$
or $\ecomp{\texttt{attr}_1}{\texttt{attr}_2}$. Attributes include background
color, digit color, digit identity and stroke type. The composition function
$\rcomp{}{}$ is addition and the distance $\rdist(\repr, \repr')$ is cosine
similarity $1 - \repr^\top \repr' / (\| \repr \| \| \repr' \|)$. 

\paragraph{Evaluation}
The training dataset consists of 9000 image triplets, evenly balanced between
positive and negative classes, with a validation set of 500 examples. At
convergence, the model achieves validation accuracy of $75.2\%$ on average over
ten training runs. (Perfect accuracy is not possible because the true classifier
is not fully determined by two training examples).  We explore the relationship
between the information bottleneck and compositionality by comparing
$\err(\ospace)$ to the mutual information $I(\repr; \obs)$ between
representations and inputs over the course of training. Both
quantities are computed on the validation set, calculating $\err(\ospace)$ as
described in \autoref{sec:tre} and $I(\repr; X)$ as described in
\citet{ShwartzZiv17Info}. (For discussion of limitations of this approach to
computing mutual information between inputs and representations, see
\citeauthor{Saxe18Info}, \citeyear{Saxe18Info}.)

\autoref{fig:info} shows the relationship between $\err(\ospace)$ and
$I(\repr;X)$. Recall that small \err is indicative of a high degree of
compositionality. It can be seen
that both mutual information and reconstruction error are initially low (because
representations initially encode little about distinctions between inputs). Both
increase over the course of training, and decrease together after mutual
information reaches a maximum (\autoref{fig:info}a). This pattern holds if we
plot values from multiple training runs at the same time (\autoref{fig:info}b),
or if we consider only the postulated compression phase (\autoref{fig:info}c).
These results are consistent with the hypothesis that compression in the
information bottleneck framework is associated with the discovery of
compositional representations.

\begin{figure}
  \centering
  \hspace{-10mm}
  \includegraphics[width=0.32\textwidth]{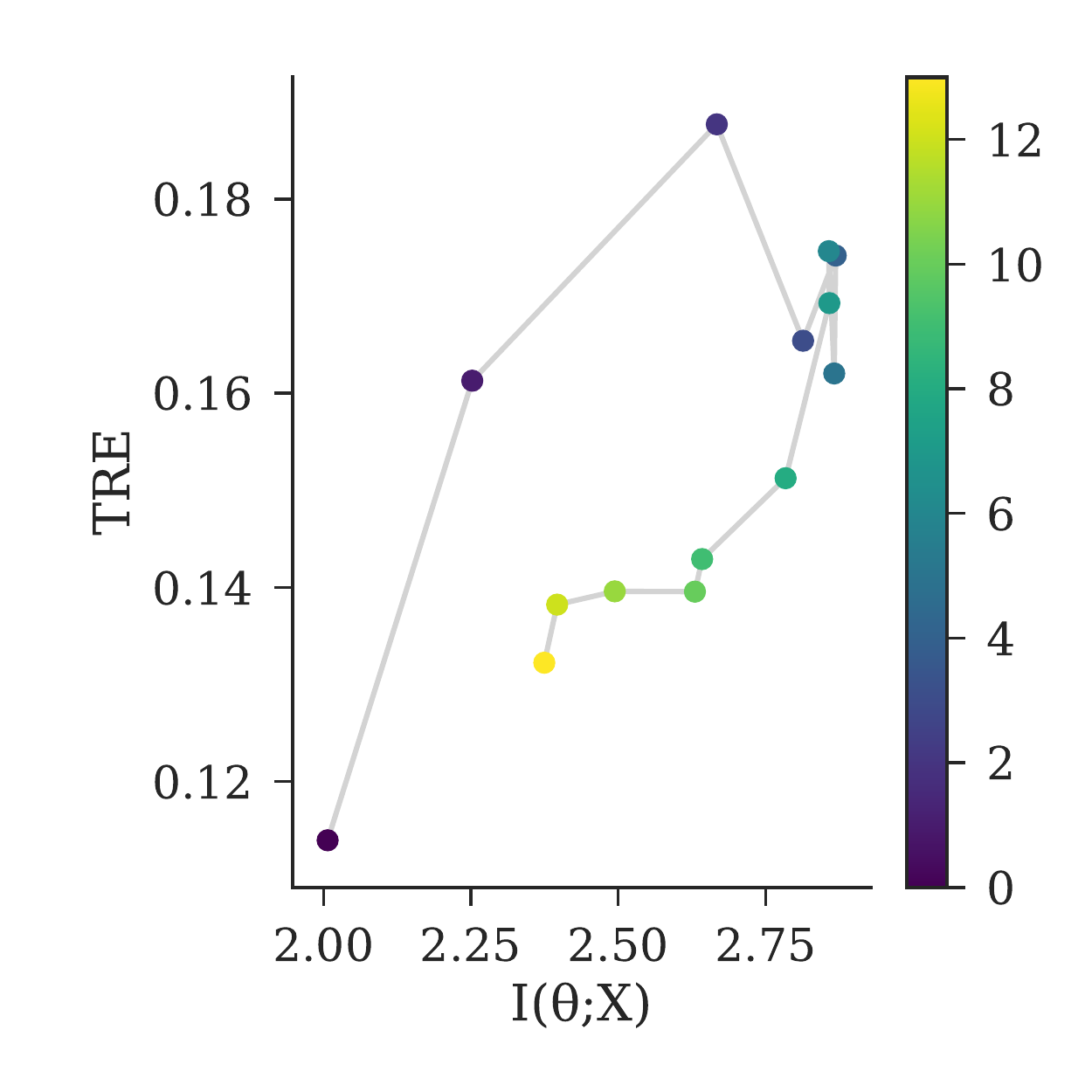}
  \includegraphics[width=0.32\textwidth]{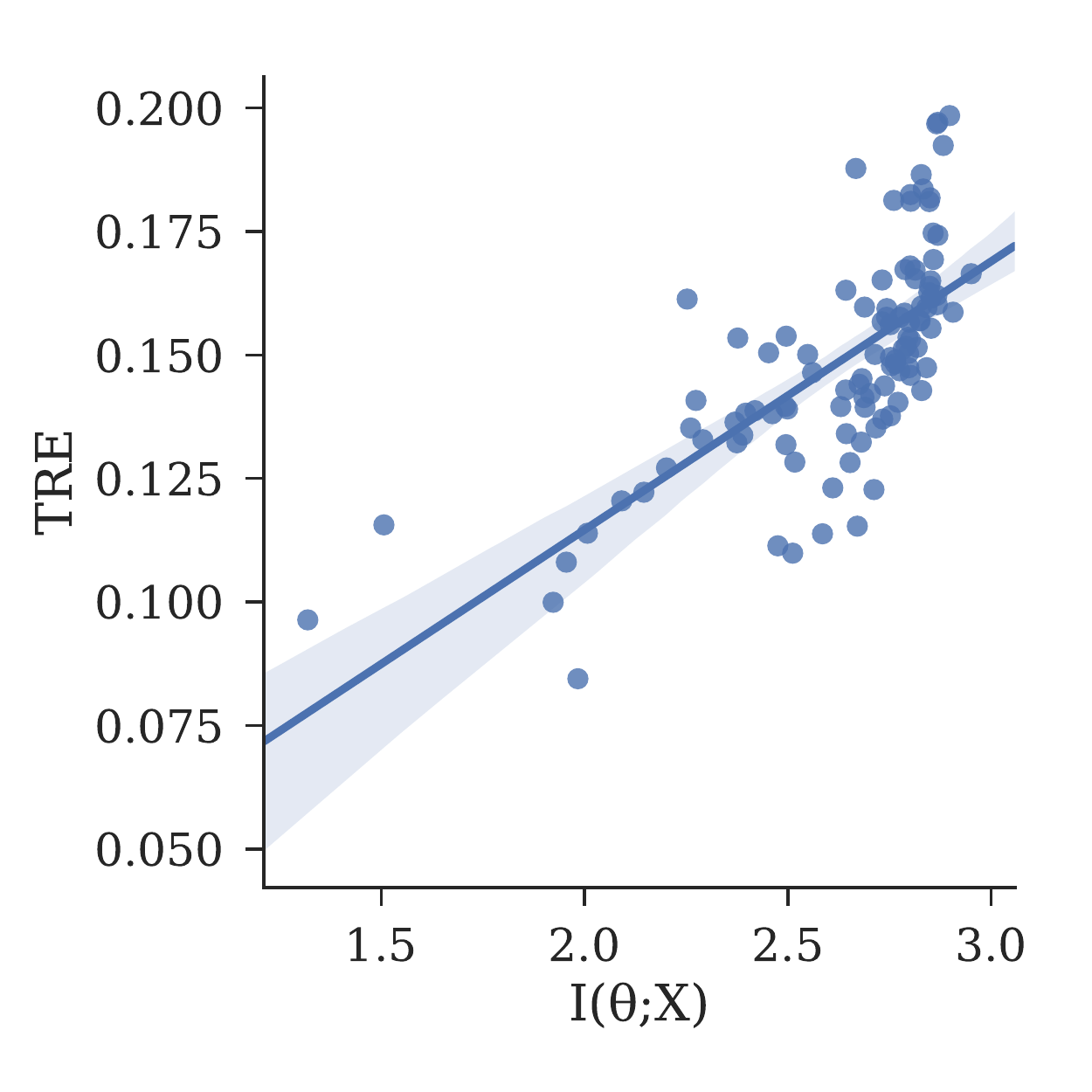}
  \includegraphics[width=0.32\textwidth]{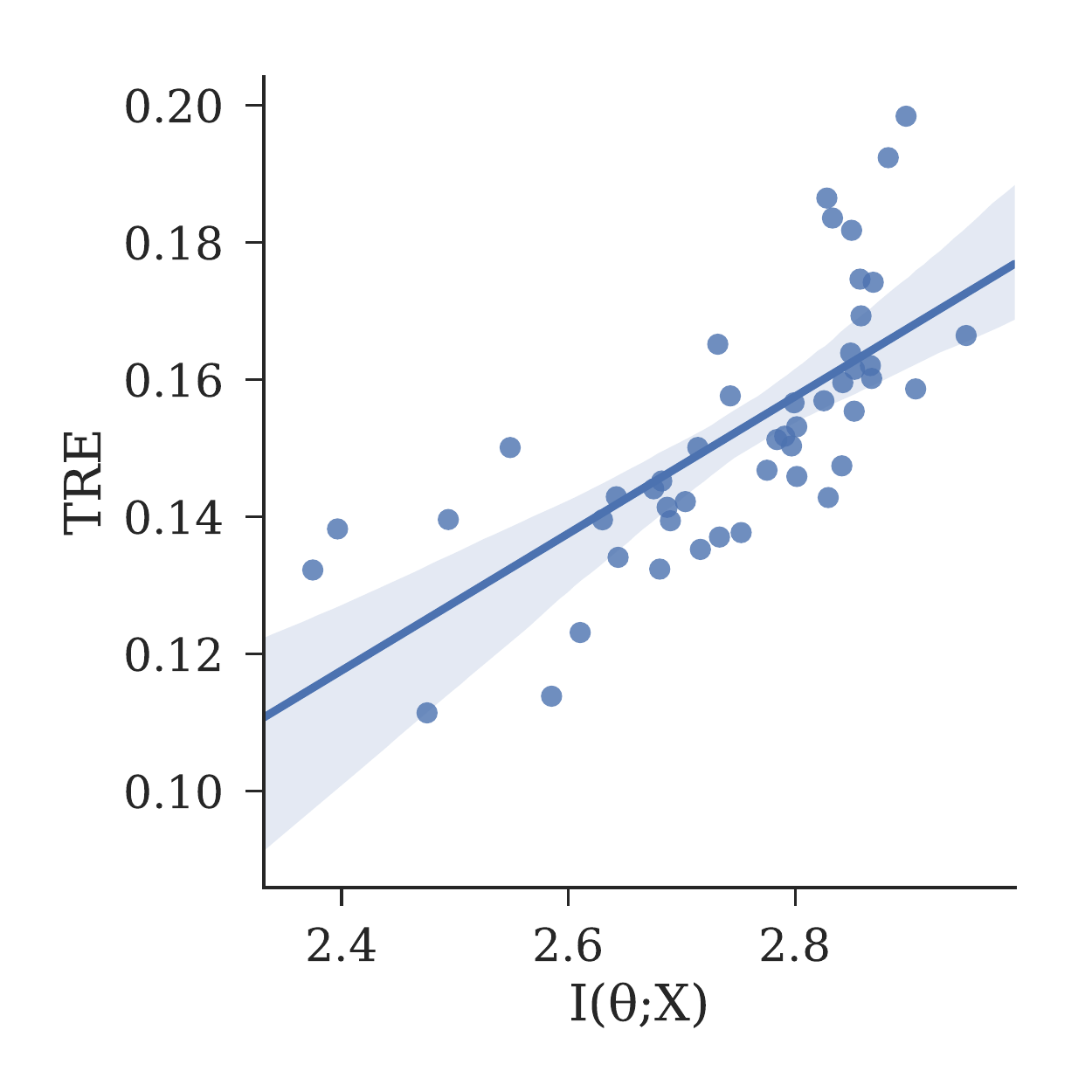} \\[-.5em]
  \strut \hfill (a) \hfill \hspace{3mm} \hfill (b) \hfill \hspace{1mm} \hfill (c) \hfill
  \hspace{2mm} \strut \vspace{-.5em}
  \caption{
    Relationship between reconstruction error $\err$ and mutual
    information $I(\repr;X)$ between inputs and representations. (a) Evolution
    of the two quantities over the course of a single run. Both initially
    increase, then decrease. The color bar shows the training epoch. (b) Values
    from ten training runs. (c) Values from the second half of each training
    run, taken to begin when $I(\repr;X)$ reaches a maximum. In (b) and (c), the
    observed correlation is significant: respectively ($r=0.70$, $p <
    1e\!-\!10$) and ($r=0.71$, $p < 1e\!-\!8$).
  }
  \label{fig:info}
  \vspace{-1em}
\end{figure}

\section{Compositionality and Human Judgments}
\label{sec:humans}

Next we investigate a more conventional representation learning task.
High-dimensional embeddings of words and phrases are useful for many natural
language processing applications \citep{Turian10Representations}, and many
techniques exist to learn them from unlabeled text
\citep{Deerwester90LSA,Mikolov13Embeddings}.  The question we wish to explore is
not whether phrase vectors are compositional in aggregate, but rather how
compositional individual phrase representations are.  Our hypothesis is that
bigrams whose representations have low $\err$ are those whose
\emph{meaning} is essentially compositional, and well-explained by the
constituent words, while bigrams with large reconstruction error will correspond
to non-compositional multi-word expressions \citep{Nattinger92MWE}.

This task is already well-studied in the natural language processing literature
\citep{Salehi15Comp}, and the analysis we present differs only in the use of
\err to search for atomic representations rather than taking them to be given by
pre-trained word representations. Our goal is to validate our approach in a
language processing context, and show how existing work on compositionality (and
representations of natural language in particular) fit into the more general
framework proposed in the current paper.

We train embeddings for words and bigrams using the CBOW objective of
\citet{Mikolov13Embeddings} using the implementation provided in FastText
\citep{Bojanowski17FastText} with 100-dimensional vectors and a context size of
5. Vectors are estimated from a 250M-word subset of the Gigaword dataset
\citep{Parker11Gigaword}. More details are provided in \autoref{app:models}.

\paragraph{Compositional structure} 

We want to know how close phrase embeddings are to the composition of their
constituent word embeddings. We define derivations for words and phrases in the
natural way: single words $w$ have primitive derivations $\expr = w$; bigrams
$w_1 w_2$ have derivations of the form $\ecomp{w_1}{w_2}$.  The composition
function is again vector addition and distance is cosine distance. 
(Future work might explore learned composition functions as in e.g.\
\citeauthor{Grefenstette13Regression}, \citeyear{Grefenstette13Regression}, for
future work.)
We compare
bigram-level judgments of compositionality computed by \err with a dataset of
human judgments about noun--noun compounds \citep{Reddy11Comp}. In this dataset,
humans rate bigrams as compositional on a scale from 0 to 5, with highly
conventionalized phrases like \emph{gravy train} assigned low scores and
\emph{graduate student} assigned high ones.

\paragraph{Results}

We reproduce the results of \citet{Salehi15Comp} within the tree reconstruction
error framework: for a given $\obs$, $\err(\obs)$ is anticorrelated with human
judgments of compositionality ($\rho=-0.34$, $p<0.01$). Collocations rated
``most compositional'' by our approach (i.e.\ with lowest \err) are:
\emph{application form}, \emph{polo shirt}, \emph{research project}; words rated
``least compositional'' are \emph{fine line}, \emph{lip service}, and \emph{nest
egg}.

\section{Compositionality and Similarity}
\label{sec:similarity}

The next section aims at providing a formal, rather than experimental,
characterization of the relationship between \err and another perspective on
the analysis of representations with help from oracle derivations.
\citet{Brighton06Topo} introduce a notion of \emph{topographic similarity},
arguing that a learned representation captures relevant domain structure if
distances between learned representations are correlated with distances between
their associated derivations. This can be viewed as providing a weak form of
evidence for compositionality---if the distance function rewards pairs of
representations that share overlapping substructure (as might be the case with
\eg string edit distance), edit distance will be expected to correlate with some
notion of derivational similarity \citep{Lazaridou2018PixelComm}.

In this section we aim to clarify the relationship between the two evaluations.
To do this we first need to equip the space of derivations described in
\autoref{sec:background} with a distance function. As the derivations considered
in this paper are all tree-structured, it is natural to use a simple \emph{tree
edit distance} \citep{Bille05TreeEditDist} for this purpose. We claim the
following:

\begin{prop}
  \label{prop:metric}
  Let $\einterp = \einterp_{\eparam^*}$ be an approximation to $\fn$ estimated
  as in \autoref{eq:treopt}, with all $\err(\obs) \leq \eps$ for some $\eps$.
  Let $\edist$ be the tree edit distance (defined formally in \autoref{app:math},
  Definition 2), and
  let $\rdist$ be any distance on $\rspace$ satisfying the following properties:
  \begin{enumerate}
    \item
      $\rdist(\einterp(\expr_i), \einterp(\expr_j)) \leq 1$ for $\expr_i,
      \expr_j \in \eprims$
    \item
      $\rdist(\einterp(\expr), 0) \leq 1$ for $\expr \in \eprims$, where $0$ is
      the identity element for $\rcomp{}{}$.
    \item
      $\rdist(\rcomp{\repr_i}{\repr_j}, \rcomp{\repr_k}{\repr_\ell}) \leq
      \rdist(\repr_i, \repr_k) + \rdist(\repr_j, \repr_\ell)$. \\(This condition
      is satisfied by any translation-invariant metric.)
  \end{enumerate}

  Then $\edist$ is an approximate upper bound on $\rdist$: for any $\obs$, $\obs'$ with $\expr =
  \oracle(\obs)$, $\expr' = \oracle(\obs')$,
  \begin{equation}
  \delta(\fn(\obs), \fn(\obs')) \leq \edist(\expr, \expr') + 2\eps \ . 
  \end{equation}
\end{prop}
In other words, representations cannot be much farther apart than the
derivations that produce them. Proof is provided in \autoref{app:math}. 

We emphasize that small $\err$ is not a sufficient condition for topographic
similarity as defined by \citet{Brighton06Topo}: very different derivations
might be associated with the same representation (\eg when representing
arithmetic expressions by their results). But this result does demonstrate that
compositionality imposes some constraints on the inferences that can be drawn
from similarity judgments between representations.

\section{Compositionality and Generalization}
\label{sec:generalization}

In our final set of experiments, we investigate the relationship between
compositionality and generalization. Here we focus on communication games like
the one depicted in \autoref{fig:teaser} and in more detail in
\autoref{fig:comm}. As in the previous section, existing work argues for a
relationship between compositionality and generalization, claiming that agents
need compositional communication protocols to generalize to unseen referents
\citep{Kottur17Takedown,Choi18Obverter}.  Here we are able to evaluate this
claim empirically by training a large number of agents from random initial
conditions, measuring the compositional structure of the language that emerges,
and seeing how this relates to their performance on both familiar and novel
objects.

\begin{wrapfigure}{r}{0.35\textwidth}
  \vspace{-2.5em}
  \centering
  \strut
  \includegraphics[width=0.3\textwidth,clip,trim=0.5cm 8.8cm 19.2cm 0cm]{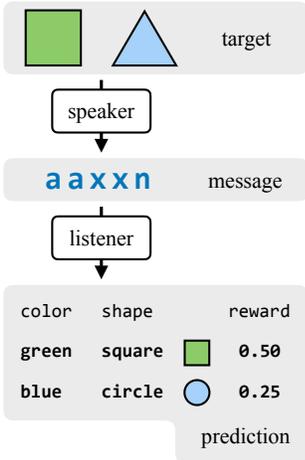}
  \vspace{-.5em}
  \caption{
    The communication task: A \emph{speaker} model observes a pair of target
    objects, and
    sends a description of the objects (as a discrete code) to a \emph{listener}
    model. The listener attempts to reconstruct the targets, receiving
    fractional reward for partially-correct predictions.
  }
  \label{fig:comm}
  \vspace{-2.7em}
\end{wrapfigure}

Our experiment focuses on a \emph{reference game} \citep{Gatt07TUNA}. 
Two policies are trained: a speaker
and a listener. The speaker observes pair of \emph{target} objects represented
with a feature vector. The speaker then sends a message (coded as a
discrete character sequence) to the listener model. The listener observes this
message and attempts to reconstruct the target objects by predicting a sequence
of attribute sets. If all objects are predicted correctly, both the speaker and
the listener receive a reward of 1 (partial credit is awarded for partly-correct
objects; \autoref{fig:comm}).

Because the communication protocol is discrete, policies are jointly trained
using a policy gradient objective \citep{Williams92Reinforce}. The speaker and
listener are implemented with RNNs; details are provided in
\autoref{app:models}.

\paragraph{Compositional structure}
Every target referent consists of two objects; each object has two attributes.
The derivation associated with each communicative task thus has the tree
structure
$\ecomp{\ecomp{\texttt{attr}_{1a}}{\texttt{attr}_{1b}}}{\ecomp{\texttt{attr}_{2a}}{\texttt{attr}_{2b}}}$.
We \emph{hold out} a subset of these object pairs at training time
to evaluate generalization: in each training run, 1/3 of possible reference
candidates are never presented to the agent at training time.

Where the previous examples involved a representation space of real embeddings,
here representations are fixed-length discrete codes. Moreover, the
derivations themselves have a more complicated semantics than in Sections
\ref{sec:bottleneck} and \ref{sec:humans}: order matters, and a commutative
operation like addition cannot capture the distinction between
$\ecomp{\ecomp{\texttt{green}}{\texttt{square}}}{\ecomp{\texttt{blue}}{\texttt{triangle}}}$
and
$\ecomp{\ecomp{\texttt{green}}{\texttt{triangle}}}{\ecomp{\texttt{blue}}{\texttt{square}}}$.
We thus need a different class of composition and distance operations. We
represent each agent message as a sequence of one-hot vectors, and take the
error function $\rdist$ to be the $\ell_1$ distance between vectors.  The
composition function has the form:
\begin{equation}
  \rcomp{\theta}{\theta'} = A\theta + B\theta'
\end{equation}
with free composition parameters $\eparam_* = \{A, B\}$ in
\autoref{eq:treopt}. Viewing each message as a matrix whose rows are token count
histograms, this composition function can rearrange the tokens in $\theta$ and
$\theta'$ across different positions of the input string, but cannot affect the
choice of the tokens themselves; this makes it possible to model non-commutative
aspects of string production.
To compute $\err$ via gradient descent, we allow the elements of $\eprims$ to be
arbitrary vectors (intuitively assigning fractional token counts to string
indices) rather than restricting them to one-hot indicators. With this change,
both $\rdist$ and $\rcomp{}{}$ have subgradients and can be optimized using the
same procedure as in preceding sections.

\begin{figure}[t]
  \strut\hfill
  \includegraphics[width=0.32\textwidth]{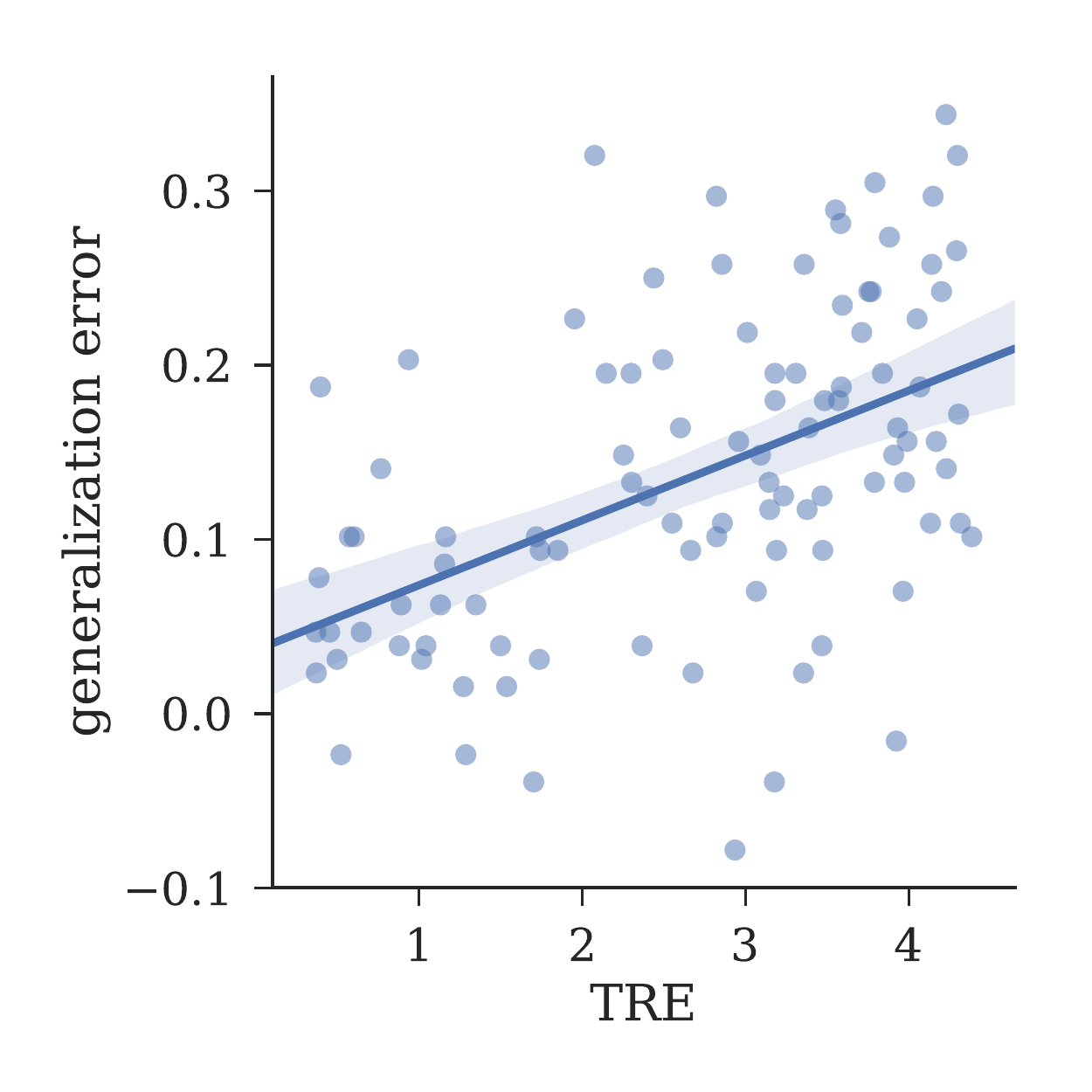}
  \hfill
  \includegraphics[width=0.32\textwidth]{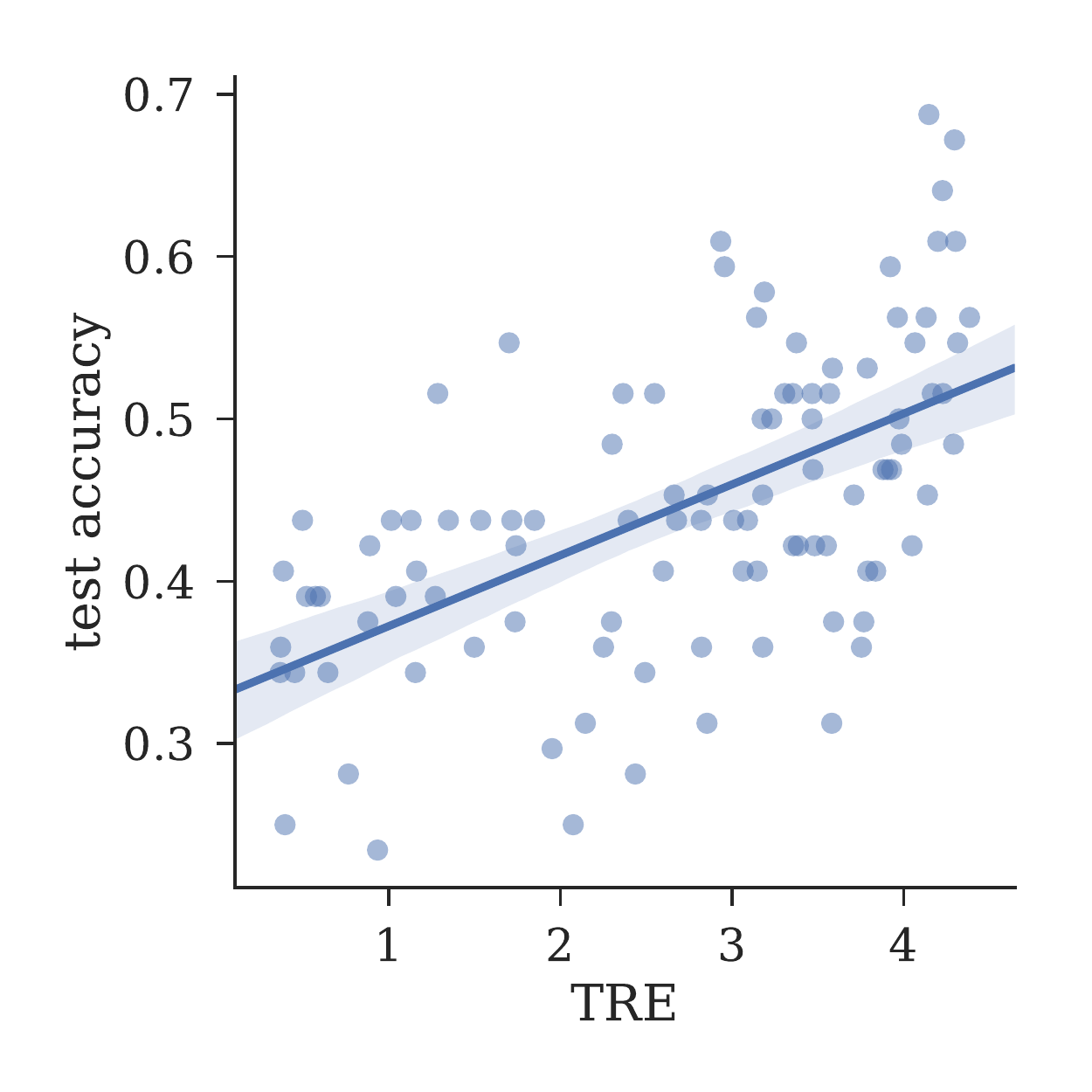} \hfill \strut \\[-.5em]
  \strut \hspace{39.5mm} (a) \hspace{56.5mm} (b)
  \caption{
    Relationship between \err and reward. (a) Compositional languages
    exhibit lower generalization error, measured as the difference between train
    and test reward ($r=0.50$, $p<1e\!-\!6$). (b) However, compositional
    languages also exhibit lower absolute performance ($r=0.57$, $p<1e\!-\!9$).
    Both facts remain true even if we restrict analysis to ``successful''
    training runs in which agents achieve a reward $> 0.5$
    on held-out referents ($r=0.6$, $p<1e\!-\!3$ and $r=0.38$, $p<0.05$
    respectively).
  }
  \label{fig:comm:graphs}
\end{figure}

\begin{figure}
	\centering
  {\tt\footnotesize
  \begin{tabular}{lrr}
		\rm & \rm Language A & \rm Language B \\
    \midrule
		((red circle) (blue triangle)) & jjjj & jeoo  \\
		((red circle) (blue star))     & oppp & jjjj  \\
		((red circle) (blue circle))   & oopp & jjjj  \\
		((red circle) (blue square))   & oopp & jjjb  \\
		((red square) (blue triangle)) & jjjj & jbjj  \\
		((red square) (blue star))     & oooo & jbjj  \\
		((red square) (blue circle))   & oooo & jbbb  \\
		((red square) (blue square))   & oooo & jbbb  \\
		\midrule
    \rm $\err$     & \rm 4.30 & \rm 2.96 \\
    \rm Train reward & \rm 0.78 & \rm 0.75 \\
    \rm Test reward  & \rm 0.61 & \rm 0.59 \\
  \end{tabular}
	}
  \caption{
    Fragment of languages resulting from two multiagent training runs. In the
    first section, the left column shows the target referent, while the
    remaining columns show the message generated by speaker in the given
    training run after observing the referent. The two languages have
    substantially different $\err$, but induce similar listener performance
    (\emph{Train} and \emph{Test reward}).
  }

  \label{fig:comm:results}
\end{figure}

\paragraph{Results}

We train 100 speaker--listener pairs with random initial parameters and measure
their performance on both training and test sets. Our results suggest a more
nuanced view of the relationship between compositionality and generalization
than has been argued in the existing literature. \err is significantly
correlated with generalization error (measured as the difference between
training accuracies, \autoref{fig:comm:graphs}a). However, \err is also
significantly correlated with absolute model reward
(\autoref{fig:comm:graphs}b)---``compositional'' languages more often result
from poor communication strategies than successful ones. This is largely a
consequence of the fact that many languages with low \err correspond to trivial
strategies (for example, one in which the speaker sends the same
message regardless of its observation) that result in poor overall performance. 

Moreover, despite the correlation between \err and generalization error, low
\err is by no means a necessary condition for good generalization. We can use
our technique to automatically mine a collection of training runs for languages
that achieve good generalization performance at both low and high levels of
compositionality. Examples of such languages are shown in
\autoref{fig:comm:results}.

\section{Conclusions}

We have introduced a new evaluation method called \err for generating graded
judgments about compositional structure in representation learning problems
where the structure of the observations is understood.  \err infers a set of
primitive meaning representations that, when composed, approximate the observed
representations, then measures the quality of this approximation.  We have
applied \err-based analysis to four different problems in representation
learning, relating compositionality to learning dynamics, linguistic
compositionality, similarity and generalization.

Many interesting questions regarding compositionality and representation
learning remain open. The most immediate is how to generalize \err to the
setting where oracle derivations are not available; in this case
\autoref{eq:treopt} must be solved jointly with an unsupervised grammar
induction problem \citep{Klein04CCM}.  Beyond this, it is our hope that this
line of research opens up two different kinds of new work: better understanding
of existing machine learning \emph{models}, by providing a new set of tools for
understanding their representational capacity; and better understanding of
\emph{problems}, by better understanding the kinds of data distributions and
loss functions that give rise to compositional- or
non-compositional representations of observations.

\subsection*{Reproducibility}

Code and data for all experiments in this paper are provided at \\
\url{https://github.com/jacobandreas/tre}.

\subsection*{Acknowledgments}

Thanks to Daniel Fried and David Gaddy for feedback on an early draft of this
paper. The author was supported by a Facebook Graduate Fellowship at the time of
writing.

\bibliography{compositionality}
\bibliographystyle{iclr2019_conference}

\newpage

\appendix

\section{Modeling Details}
\label{app:models}

\paragraph{Few-shot classification}

The CNN has the following form:
\begin{verbatim}
Conv(out=6, kernel=5)
ReLU
MaxPool(kernel=2)
Conv(out=16, kernel=5)
ReLU
MaxPool(kernel=2)
Linear(out=128)
ReLU
Linear(out=64)
ReLU
\end{verbatim}
The model is trained using \textsc{adam} \citep{Kingma14Adam} with a learning
rate of .001 and a batch size of 128. Training is ended when the model stops
improving on a held-out set.

\paragraph{Word embeddings}

We train FastText \citep{Bojanowski17FastText} on the first 250 million words of
the NYT section of Gigaword \citep{Parker11Gigaword}. To acquire bigram
representations, we pre-process this dataset so that each occurrence of a bigram
from the \citet{Reddy11Comp} dataset is treated as a single word for purposes of
estimating word vectors.

\paragraph{Communication}

The encoder and decoder RNNs both use gated recurrent units \citep{Cho14GRU} with
embeddings and hidden states of size 256. The size of the discrete vocabulary is
set to 16 and the message length to 4. Training uses a policy gradient
objective with a scalar baseline set to the running average reward; this is
optimized using \textsc{adam} \citep{Kingma14Adam} with a learning rate of .001
and a batch size of 256. Each model is trained for 500 steps. Models are trained
by sampling from the decoder's output distribution, but greedy decoding is used
to evaluate performance and produce \autoref{fig:comm:results}.

\section{Proposition \ref{prop:metric}}
\label{app:math}

First, some definitions:\\

\begin{defn}
  The {\bfseries size} of a derivation is given by:
  \begin{align}
    |\expr| &= 1 \hspace{5.8em} \textrm{if $\expr \in \eprims$} \nonumber \\
    |\ecomp{\expr_a}{\expr_b}| &= |\expr_a| + |\expr_b| \qquad \textrm{otherwise}
  \end{align}
\end{defn}

\begin{defn}
  The {\bfseries tree edit distance} between derivations is defined by:
  \begin{align}
    \edist(\expr_i, \expr_j) &= \indicate[i = j] \quad \textrm{if $\expr_i \in \eprims$ and
    $\expr_j \in \eprims$} \nonumber \\[1mm]
  \edist(\expr_i, \ecomp{\expr_j}{\expr_k}) &= \min \left\{ \begin{array}{c}
        \edist(\expr_i, \expr_j) + |\expr_k| \\[1mm]
        \edist(\expr_i, \expr_k) + |\expr_j| \\[1mm]
    \end{array} \right\} \quad \textrm{if $\expr_i \in \eprims$} \nonumber \\[1mm]
    \edist(\ecomp{\expr_i}{\expr_j}, \ecomp{\expr_k}{\expr_\ell}) &=
    \min \left\{ \begin{array}{ll}
        \edist(\expr_i, \expr_k) + \edist(\expr_j, \expr_\ell) \\[1mm]
        \edist(\ecomp{\expr_i}{\expr_j}, \expr_k) + |\expr_\ell| &
        \edist(\ecomp{\expr_i}{\expr_j}, \expr_\ell) + |\expr_k| \\[1mm]
        \edist(\ecomp{\expr_k}{\expr_\ell}, \expr_i) + |\expr_j| &
        \edist(\ecomp{\expr_k}{\expr_\ell}, \expr_j) + |\expr_i| \\[1mm]
    \end{array} \right\}
    \label{eq:edist}
  \end{align}
\end{defn}

Now, suppose we have $\obs$ and $\obs'$ with derivations $\expr =
\oracle(\obs)$, $\expr' = \oracle(\obs')$ and representations $\repr =
\fn(\obs)$, $\repr' = \fn(\obs')$.
Proposition \ref{prop:metric} claims that $\rdist(\repr, \repr') \leq
\edist(\expr, \expr') + 2\eps$. 

\begin{lemma}
  $\rdist(\einterp(\expr), 0) \leq |\expr|$. 
\end{lemma}
\begin{proof}
  For $\expr \in \eprims$ this follows immediately from Condition 2 in the
  proposition. For composed derivations it follows from Condition 3 taking
  $\repr_k = \repr_\ell = 0$ and induction on $|\expr|$.
\end{proof}

\begin{lemma}
  $\rdist(\einterp(\expr), \einterp(\expr')) \leq \edist(\expr, \expr')$
\end{lemma}

\begin{proof}
  By induction on the structure of $\expr$ and $\expr'$:

  \paragraph{Base case}
  Both $\expr, \expr' \in \eprims$.

  If $\expr = \expr'$, $\rdist(\einterp(\expr), \einterp(\expr')) =
  \rdist(\einterp(\expr), \einterp(\expr)) = 0 = \edist(\expr, \expr')$.

  If $\expr \neq \expr'$, $\rdist(\einterp(\expr), \einterp(\expr')) \leq 1 =
  \edist(\expr, \expr')$ from Condition 1.

  \paragraph{Inductive case}

  Consider the arrangement of derivations that minimizes \autoref{eq:edist} for
  derivation $\expr$ and $\expr'$. 
  There are two possibilities:

  \emph{Case 1:}
  $\edist(\expr, \expr')$ has the form $\edist(\expr_i, \expr_k) +
  \edist(\expr_j, \expr_\ell)$ for some $\expr_{i,j,k,\ell}$. W.l.o.g.\ let
  $\expr = \ecomp{\expr_i}{\expr_j}$ and $\expr' = \ecomp{\expr_k}{\expr_\ell}$.
  Then,
  \begin{align*}
    \rdist(\einterp(\expr), \einterp(\expr'))
    &= \rdist(\rcomp{\einterp(\expr_i)}{\einterp(\expr_j)}, \rcomp{\einterp(\expr_k)}{\einterp{\expr_\ell}}) \\
    &\leq \rdist(\einterp(\expr_i), \einterp(\expr_k)) + \rdist(\einterp(\expr_j), \einterp{\expr_\ell}) \\
    &\leq \edist(\expr_i, \expr_k) + \edist(\expr_j, \expr_\ell) \\
    &= \edist(\expr, \expr')
  \end{align*}

  \emph{Case 2:}
  $\edist(\expr, \expr')$ has the form $\edist(\expr_i, \expr_k) + |\expr_j|$
  for some $\expr_{i,j,k}$. W.l.o.g.\ let $\expr = \ecomp{\expr_i}{\expr_j}$ and
  $\expr' = \expr_k$. Abusing notation slightly, let us define $\edist(\expr, 0)
  = |\expr|$. If we let $\expr_\ell = 0$ this case reduces to the previous one.
\end{proof}

Finally,

\begin{proof}[Proof of Proposition \ref{prop:metric}]
  \begin{align*}
    \rdist(\repr, \repr') &\leq \rdist(\einterp(\expr), \einterp(\expr')) + 2\eps \\
    &\leq \edist(\expr, \expr') + 2\eps
  \end{align*}
\end{proof}
  
\end{document}